%% file: main.tex
\documentclass[sn-mathphys-num]{sn-jnl}
\usepackage[title]{appendix}%
\usepackage{lmodern}
\usepackage{amsmath,amsthm}
\usepackage{amssymb}
\usepackage{amsfonts}
\usepackage{amscd}
\usepackage{booktabs}
\usepackage{color}
\usepackage{subcaption}

\newtheorem{thm}{Theorem}[section]
\newtheorem{theorem}[thm]{Theorem}
\newtheorem{prop}[thm]{Proposition}

\newtheorem{lemma}[thm]{Lemma}

\theoremstyle{definition}

\newtheorem{definition}[thm]{Definition}
\newtheorem{defi}[thm]{Definition}

\newtheorem{remark}[thm]{Remark}
\newtheorem{observation}[thm]{Observation}

\newcommand{\R}{\mathbb{R}}

\newcommand{\E}{\mathbb{E}}

\newcommand{\X}{\mathcal{X}}

\newcommand{\Y}{\mathcal{Y}}
\newcommand{\T}{T}

\newcommand{\cL}{\mathcal{L}}

\newcommand{\cC}{\mathcal{C}}
\newcommand{\cD}{\mathcal{D}}
\newcommand{\D}{\mathcal{D}}
\newcommand{\cF}{\mathcal{F}}

\newcommand{\beq}{\begin{equation}}
\newcommand{\eeq}{\end{equation}}
\newcommand{\pa}[1]{\left( #1 \right)}
\newcommand{\bra}[1]{\left[ #1 \right]}
\newcommand{\set}[1]{\left\{ #1 \right\}}
\newcommand{\ang}[1]{\left<#1\right>}
\newcommand{\Span}{\mathrm{span}}

\newcommand{\diag}{\mathrm{diag}}
\newcommand{\re}{\mathrm{e}}
\newcommand{\Hess}[1]{{H}\pa{#1}}
\newcommand{\der}[2]{\frac{\partial #1}{\partial #2}}
\newcommand{\lra}{\longrightarrow}

\newcommand{\nicolas}[1]{{#1}}
\newcommand{\eliot}[1]{{#1}}

\newcommand{\revision}[1]{{#1}}

\begin{document}
\title{Cartan moving frames and the data manifolds}

\author*[1]{\fnm{Eliot} \sur{Tron}}\email{eliot.tron@enac.fr}

\author[3]{\fnm{Rita} \sur{Fioresi}\email{rita.fioresi@UniBo.it}}

\author[1,2]{\fnm{Nicolas} \sur{Couellan}}\email{nicolas.couellan@recherche.enac.fr}

\author[1,2]{\fnm{St\'{e}phane} \sur{Puechmorel}}\email{stephane.puechmorel@enac.fr}

\affil[1]{\orgname{Ecole Nationale de l'Aviation Civile}, \orgaddress{\street{7 Avenue Edouard Belin}, \city{Toulouse}, \postcode{31400}, \country{France}}}

\affil[2]{\orgname{Institut de Mathématiques de Toulouse}, \orgdiv{UMR 5219}, Universit\'{e} de Toulouse, CNRS, UPS, \orgaddress{\street{118 route de Narbonne}, \city{Toulouse}, \postcode{F-31062} Cedex 9, \country{France}}}

\affil[3]{\orgname{FaBiT, Università di Bologna}, \orgaddress{\street{via S. Donato 15}, \postcode{I-40126} \city{Bologna}, \country{Italy}}}




\abstract{
The purpose of this paper is to employ the language of Cartan
moving frames to study the geometry of the data manifolds and
its Riemannian structure, via the data information metric
and its curvature at data points. \nicolas{Using this framework and through experiments, explanations on the response of a neural network are given by pointing out the output classes that are easily reachable from a given input. This emphasizes how the proposed mathematical relationship between the output of the network and the geometry of its inputs can be exploited as an explainable artificial intelligence tool.} 
}

\maketitle

\section{Introduction}\label{intro-sec}

In machine learning, the idea of data manifold is based on the assumption that the data space,
containing the data points on which we perform classification tasks,
has a natural Riemannian manifold structure. It is a quite old concept (see \cite{fefferman2016}, \cite{Roweis2000}
and refs therein) and it is linked to the key
question of dimensionality reduction \cite{maaten},
which is the key to efficient data processing
for classification tasks, and more. 
Several geometrical tools, such as geodesics, connections,
Ricci curvature, become readily available for machine learning problems
under the manifold hypothesis, and are especially effective
when employed in examining adversarial attacks \cite{shenDefendingAdversarialAttacks2019,
martinInspectingAdversarialExamples2019,
carliniAdversarialExamplesAre2017,zhaoAdversarialAttackDetection2019, Tron2022CanonicalFO}
or knowledge transfer questions,
see \cite{Weiss2016ASO, Pan2010ASO} and refs. therein.

More specifically, 
the rapidly evolving discipline of information geometry \cite{Sun2014AnIG,
amariDifferentialGeometryStatistical1987, amari2016information,
nielsenElementaryIntroductionInformation2020}
is now offering methods for discussing the above questions, once
we cast them appropriately by relating the statistical manifold, i.e.
the manifold of probability measures, studied in information geometry
with data manifolds (see \cite{Sun2014AnIG, Grementieri_Fioresi_2022}). \revision{Data manifolds have also been studied from a Riemannian geometric perspective by pullbacks of the Fisher information metric on the probability simplex through simple families of functions in the field of {metric learning} \cite{Lebanon_2012, Wang_Sun_Sha_Marchand-Maillet_Kalousis}.

In this paper, }we are interested in a naturally emerging foliation structure on the data space
coming through the {\sl data information matrix} (DIM), which is the analog of
the Fisher information matrix and in concrete experiments can be obtained by looking
at a 
deep learning neural network for classification tasks. 
As it turns out,
the leaves of such foliation are related with the dataset the network was
trained with \cite{Grementieri_Fioresi_2022, tronfioresi}.
Further work in \cite{tronfioresi} linked such study to
the possible applications to knowledge transfer. 

The purpose of the present paper is to study and understand the {\sl data manifolds}, with Cartan moving frames method.
Following the philosophy in \cite{Sun2014AnIG}, we want to equip the manifolds
coming as leaves of the above mentioned foliation, via a natural metric coming through the
data information matrix. As it turns out, the partial derivatives of the probabilities allow us to naturally
define the Cartan moving frame at each point and are linked with the curvature
of such manifold in an explicit way.

\nicolas{From a broader perspective, the work proposed here emphasizes the mathematical relationship between the changes in the outputs of a neural network and the curvature of the data manifolds. Furthermore, we show how this relationship can be exploited to provide explanations for the responses of a given neural network. In critical systems, providing explanations to AI model decisions can sometimes be required or mandatory by certification processes (see for example \cite{AIR6988}). More generally, the field of eXplainable Artificial Intelligence (AI) is a fast growing research topic that develops tools for understanding and interpreting predictions given by AI models \cite{samek2019}. In Section~\ref{sec:exp} of this article, through simple experiments, we show how 
the DIM restricted to the moving frame given by the neural network output probability partial derivatives can be used to understand the local geometry of trained data.
Specifically, in the case of the MNIST handwritten digits dataset \cite{lecun1998mnist}
 and CIFAR10 animals / vehicles dataset \cite{cifar-10} 
by displaying the restricted DIM in the form of images, we are able to understand, starting from a given data point, which classes are easily reachable by the neural networks or not.
}

\medskip
The organization of the paper is as follows.

In Section~\ref{sec:preliminaries}, we briefly recap some notions on information geometry and
some key known results that we shall need in the sequel. Our main reference
source will be \cite{amari2016information,
nielsenElementaryIntroductionInformation2020}.

In Section~\ref{sec:cartan-moving-frames}, we take advantage of the machinery developed by Cartan
(see \cite{tu2017differential,Jost2008RiemannianGA}) and define moving frames
on the data manifolds via the partial derivatives of the probabilities.

In Section~\ref{sec:curv-of-data-leaves} and \ref{sec:computation-curv-forms}, we relate the probabilities with the curvature of the leaves, by deriving the calculations of the curvature forms in a numerically stable way.

Finally in Section~\ref{sec:exp}, we consider some experiments on MNIST and CIFAR10 elucidating how
near a data point, some partial derivatives of the probabilities become more important and the metric
exhibits some possible singularities.

\bigskip
{\bf Acknowledgements.}
We thank Emanuele Latini, Sylvain Lavau for 
helpful discussions. This research was supported by Gnsaga-Indam, by COST Action
CaLISTA CA21109, HORIZON-MSCA-2022-SE-01-01 CaLIGOLA, MSCA-DN CaLiForNIA - 101119552,
PNRR MNESYS, PNRR National Center for HPC, Big Data and
Quantum Computing, INFN Sezione Bologna. ET is grateful to the FaBiT Department of
the University of Bologna for the hospitality.


\section{Preliminaries}
\label{sec:preliminaries}

The {\it neural networks} we consider 
are classification functions 
\[
N:\X \times \Theta \lra \set{p\in\R^C ~;~\sum_k p_k = 1,~ p_k > 0},\]
from the spaces of data $\X$ and weights (or parameters) $\Theta$,
to the space of parameterized probability densities
$p\pa{y\mid x,\theta}$ over the set of labels $\Y$, where
$x\in\X$ is the input,  $y\in \Y$ is the target label
and $\theta\in \Theta$ is the parameter of the model.
For instance $\theta$ are the weights and biases in a perceptron model
predicting $C$ classes, for a given input datum $x$.

We may assume both the dataspace $\X$ and the parameter space
$\Theta$ to be (open) subsets of euclidean spaces,
$\X \subset \R^d$, $\Theta \subset \R^n$, though it is clear
that in most practical situations only a tiny portion of
such open sets will be effectively occupied by data points or parameters
of an actual model.

In the following, we make only two assumptions on $N$:
\begin{align}
    & N = \mathrm{softmax} \circ s \tag{H1}\label{H1}\\
    \forall i,j,k, \quad & \partial_{x_i} \partial_{x_j} s_k(x) = 0,  \qquad \revision{\forall x\in \X\setminus A}\tag{H2}\label{H2}
\end{align}
where $s$ is called a \emph{score} function, \revision{$A$ is a nowhere dense set} and $\mathrm{softmax}$ is the function $\pa{a_i}_i\mapsto \pa{e^{a_i} / \sum_k e^{a_k}}_i$. We detail why these assumptions are important in the following sections. 

\revision{Examples of neural networks that satisfies these assumptions are feed-forward multi-layer perceptrons, or convolutional neural networks, with ReLU activation functions or variants such as leaky ReLU, or any piecewise linear activation functions. These kind of architectures are popular within the machine learning community.}


\medskip
A most important notion in information geometry is the Fisher-Rao matrix $F$, providing
in some important applications, a metric on the space $\Theta$:
\[
F=(f_{ij}):=\left(\E_{y \mid x,\theta}\bra{\partial_{\theta_i} {\ln p(y\mid x,\theta)}
\partial_{\theta_j} {\ln p(y\mid x,\theta)}}\right), \qquad 1\leq i,j\leq n 
.\]
We now give in analogy to $F$ the {\sl data information matrix} $D$ (DIM for short), defined in
a similar way.

\medskip
\begin{defi} \label{dim-def}
{\rm
We define \textit{data information matrix} $D=(D_{ij})$ for a point
$x \in \X$ and a fixed model $\theta \in \Theta$ to be
the following symmetric matrix:
\[
D_{i j}(x) = \E_{y \mid x,\theta}\bra{\partial_{x_i} {\ln p(y\mid x,\theta)}
\partial_{x_j} {\ln p(y\mid x,\theta)}}, \qquad 1\leq i,j\leq d
.\]
}
\end{defi}

\revision{In the data information matrix, the derivative is taken with respect to the input $x$ and not with respect to the parameter $\theta$ as in the Fisher matrix.} Hence, $D$ is a $d \times d$ matrix - $d$ the dimension of the dataspace - \revision{which makes it possible the study of the relationship between input-output changes in the neural network.}\footnote{\revision{Note that $D$ can be retrieved as the pullback of the Fisher metric on the $(C-1)$-simplex by the neural network $N$.}}
\begin{observation}
In what follows, we will use the notation $p_k(x) := p\pa{y_k\mid x,\theta}$ in the case of classification into a finite set $\Y = \set{y_1,\ldots,y_C}$. We then concatenate these values in the vector $p(x) = \pa{p_k(x)}_{k=1,\ldots,C}$. In practice, $p(x) = N(x)$.

We omit the dependence in $\theta$ as the parameters remain unchanged during the rest of the paper. We shall omit the dependence in $x$ too in long computations for easier reading.
\end{observation}

\medskip
As one can readily check, we have that:
\beq
D_{ij}(x) =\sum_k \frac{1}{p_k(x)} \partial_{x_i} p_k(x) \partial_{x_j} p_k(x)
.\eeq

In the following, we will use $\partial_{i}$, the canonical basis of $\T\X$ associated to the coordinates $x^i$, and the Einstein summation notation.
\medskip
\begin{remark} {\rm Notice the appearance of the probability $p_i$
at the denominator in the expression of $D$. Since $p(y\mid x,\theta)$
is an empirical probability,
on data points it may happen that some of the $p_k(x)$ are 
close to zero, giving numerical instability in practical situations.
We shall comment on this problem and how we may solve it, later on.}
\end{remark}

We have the following result 
\cite{Grementieri_Fioresi_2022}, we briefly recap its proof, for completeness.

\medskip
\begin{thm}\label{thm:kerF}
The data information matrix $D(x)$ is positive semidefinite, moreover
\begin{equation*}
\ker D(x) = \pa{\Span_{k=1,\ldots,C}\set{\sum_i\partial_{i} \ln p_k(x) \partial_{i}}}^\perp
\end{equation*}
with $\perp$ taken w.r.t. the Euclidean scalar product on $\T\X$.
Hence,  
the rank of $D$ 
is bounded by $C-1$, with $C$ the number of classes.
\end{thm}

\begin{proof} To check semipositive definiteness, let $u\in\T_x\X$. Then,
\begin{align*}
				u^T D(x) u &= \sum_{i,j} u_i u_j \E_{y\mid x,\theta} \bra{\partial_{i} \ln p\pa{y\mid x, \theta} \partial_{j} \ln p\pa{y\mid x, \theta}}\\
									 &= \E_{y\mid x,\theta} \bra{ {\sum_i u_i \partial_{i} \ln p\pa{y\mid x, \theta}}  \sum_j u_j\partial_{j} \ln p\pa{y\mid x, \theta}}\\
									 &= \E_{y\mid x,\theta} \bra{\ang{\partial_{i} \ln p\pa{y\mid x,\theta} \partial_{i}, u}_\re^2}.
\end{align*}
For the statement regarding kernel and rank of $D(x)$, we notice that $u^T D(x) u=0$,
whenever $u \in \pa{\Span_{k=1,\ldots,C}\set{\sum_i\partial_{i} \ln p_k(x) \partial_{i}}}^\perp$,
where $\ang{\cdot, \cdot}_\re$ and $\perp$ refer to the Euclidean
metric. Besides, $\sum_k p_k =1\implies \sum_{ik} \partial_i p_k = 0 $, hence the rank bounded by $C-1$ and not $C$.
\end{proof}

This result prompts us to define the distribution:
\begin{equation}\label{distr-def}
x \mapsto \D_x := \Span \set{\sum_i\partial_{i} p_k\pa{x} \partial_{i},~ k=1, \ldots, C-1}
.\end{equation}

For popular neural networks (satisfying \ref{H1} and \ref{H2}), the distribution $\D$
turns to be integrable in an open set of $\R^d$, 
hence it defines a {\sl foliation}. 
This matter has been discussed in~\cite{tronfioresi}.

\medskip
\begin{theorem}\label{frob-thm}
Let $\theta$ be the weights of a neural network classifier $N$ satisfying \ref{H1} and \ref{H2}, associated with the vector $p$ given by softmax.
Assume $D$ has constant rank. Then, at each smooth point $x$ of $N$
there exists a local submanifold $\cL$ of $\X$, such that
its tangent space at $x$, $\T_x \cL =\cD$.
\end{theorem}

\medskip
In particular, given a dataset (e.g. MNIST),
we call {\it data leaf} a leaf of such foliation containing at least one point of
the dataset. In \cite{tronfioresi} the significance of such leaves is fully explored
from a geometrical and experimental point of view. More specifically, Thm. 3.6 in
\cite{tronfioresi} shows that the open set 
in Thm \ref{frob-thm} is dense in the dataspace $\X$.

\medskip
In the next sections we shall focus on the
geometry of this foliation in the data space.

\section{Cartan moving frames}
\label{sec:cartan-moving-frames}

In this section we first review some basic notions of Cartan's approach
to differentiable manifolds via the moving frames point of view and then
we see some concrete application to our setting. \revision{A more detailed introduction to the Cartan moving frames approach can be found in Appendix~\ref{app:cartan_moving_frames}.}


\medskip
\begin{defi} 
{\rm Let $E \to \cL$ be a  $\mathcal{C}^\infty$ vector bundle over
a smooth manifold $\cL$. A \emph{connection} on $E$ is a bilinear map:
\[
\nabla: \mathfrak{X}\pa{\cL} \times \Gamma\pa{E} \to \Gamma\pa{E}
\]
such that
\[
\nabla _{v}(fs)=df(v)s+f\nabla _{v}s, \qquad v \in \mathfrak{X}\pa{\cL},\,
f \in \cC^\infty(\cL), \, s \in \Gamma\pa{E}
\]
where $\mathfrak{X}\pa{\cL}$ are the vector fields on $\cL$ and  $\Gamma\pa{E}$
the sections of the vector bundle $E$.} 
\end{defi}

We shall be especially interested to the case, $E = \T\cL$ the tangent bundle to a leaf $\cL$ 
of a suitable foliation (Thm. \ref{frob-thm}). 
Our framework is naturally set up to take advantange of Cartan's language
of {\sl moving frames} \cite{tu2017differential}. Before we give this key definition,
let us assume our distribution $\cD$ as in (\ref{distr-def}) to be smooth, constant
rank and integrable. This assumption is reasonable, since it has been shown in~\cite{tronfioresi} that,
for neural networks with ReLU
non linearity, $\cD$ satisfies these hypotheses 
in a dense open subset of the dataspace. 

\medskip

\begin{definition}
We define the
\textit{data foliation} $\mathcal{F}_\D$, the foliation in the data space $\X$
defined by the distribution $\cD$ as in Equation~\ref{distr-def} and \textit{data leaf} a
leaf $\cL_x$ of $\cF$ containing at least one \textit{data point} i.e. a point of the data set
the network was trained with.
\end{definition}

\medskip
If $\ang{\cdot,\cdot}_D$ denotes the symmetric
bilinear form defined by the data information matrix
$D$ as in Def. \ref{dim-def} in the dataspace, by the very definitions, 
when we restrict ourselves to the tangent space to a leaf $\cL$ of $\cF$, we have
that $\ang{\cdot,\cdot}_{\cL}:=\ang{\cdot,\cdot}_D|_\cL$ is a non degenerate inner product
(Prop. \ref{thm:kerF} and Thm. \ref{frob-thm}).
Hence it defines a {\sl metric} denoted $g^D$.

\medskip
\begin{defi}
{\rm Let the notation be as above and let $\cL$ be a fixed leaf in $\cF_D$.  
At each point $x \in \cL$,
we define $\pa{e_k:=\sum_i\partial_{i} p_k(x)\partial_{i}}_{(k=1, \dots C-1)}$ 
to be a \textit{frame}
for $\T_x \cL$. The symmetric bilinear form  $g^D$ defines a Riemannian
metric on $\cL$, that we call {\it data metric}.}
\end{defi}

\medskip
Let $\nabla$ be the Levi-Civita connection with respect to the data information metric.

\medskip
\begin{defi}\label{def:connection_forms}
{\rm The Levi-Civita connection
\begin{equation}\label{lc-con}
				\nabla_X e_j = \sum_i \omega_j^i (X) e_i
\end{equation}
defines $\omega_j^i$, which are called the \emph{connection forms} and $\omega$
the \emph{connection matrix} relative to the frame $\pa{e_i}$.}
\end{defi}

\medskip
The Levi-Civita connection is explicitly given by (see \cite{tu2017differential} pg 46):
\begin{multline}
2 g^D\pa{\nabla_{e_a} e_b, e_c} = e_a\pa{g^D\pa{e_b, e_c}} + e_b \pa{g^D\pa{e_c, e_a}} - e_c\pa{g^D\pa{e_a, e_b}} \\
- g^D\pa{ e_a, \bra{e_b,e_c} } + g^D\pa{e_b, \bra{e_c,e_a}} + g^D\pa{e_c, \bra{e_a, e_b}}.
\end{multline} 


\medskip
To explicitly compute the connection forms $\omega_j^i(e_k)$ we need to define 
\begin{equation}
				C_{a,b,c} := g^D\pa{\nabla_{e_a} e_b, e_c}
\end{equation}
then we have by definition that
\[
				C_{a,b,c} = \sum_i \omega^i_b \pa{e_a} g^D\pa{e_i,e_c}
.\]

Define the matrix $\hat{D} = \pa{g^D\pa{e_i,e_j}}_{i,j=1,\ldots,C-1}$. This is the matrix of the
metric $g^D\pa{\cdot,\cdot}$ 
restricted to a leaf $\cL$ for the basis given by a frame.
We shall see an interesting significance for the matrix $\hat{D}$ in the
experiments in Section~\ref{sec:exp}.
Hence:

\[
				\sum_i \hat{D}_{l,i} \omega^i_j \pa{e_k}  = C_{k,j,l}
.\] 

This gives, in matrix notation: 
\begin{equation}
				\omega_j^i \pa{e_k} = \pa{\hat{D}^{-1} C_{k,j,\cdot}}_i.
\end{equation}

\medskip
The \textit{curvature tensor} $R$  is given by:
\[
R(X,Y)e_j = \sum_i \Omega_j^i (X,Y) e_i
\]
where $\Omega_j^i$ is a 2-form on $\T\X$, alternating and  $\D$-bilinear 
called the  \emph{curvature form (matrix)} of the connection $\nabla$ relative to the frame  $\pa{e_i}$ on  $\T\X$.

\medskip
To compute $R$ we shall make use of the following result found in 
\cite{tu2017differential} Theorem 11.1.

\medskip
\begin{prop}\label{curv-form}
The curvature form $\Omega$ is given by:
\[
\Omega^i_j \pa{X,Y}  = \pa{d\omega^i_j} \pa{X,Y} + \sum_k\omega^i_k\wedge\omega^k_j\pa{X,Y}
.\] 
\end{prop}


\medskip
By using all the previous propositions and definitions, we will see in the following
sections how we can numerically compute
the curvature for a neural network with a softmax function on the output.

\medskip
It is useful to 
recall some formulae. If
$\alpha$ and $\beta$ are $\cC^\infty$ 1-forms and $X, Y$ are $\cC^\infty$ vector fields on a manifold, then 
\beq\label{for1}
\pa{\alpha \wedge \beta}\pa{X,Y} = \alpha(X)\beta(Y) - \alpha(Y)\beta(X)
\eeq 
and 
\beq\label{for2}
\pa{d\alpha}(X,Y) = X\alpha(Y) - Y\alpha(X) - \alpha \pa{\bra{X,Y}}
.\eeq

\section{The Curvature of the Data Leaves}
\label{sec:curv-of-data-leaves}

Computing in practice the curvature of a manifold with many dimensions is often intractable though essential
for many tasks. This is why we are interested in computing the curvature just for the data leaves.
\revision{In this section and the following one, we derive the curvature form calculations to obtain a formula linking curvature and first-order derivative of the network (i.e. the frame $e_k$), without double automatic differentiation or division by $p_i$ which might be close to zero, especially on data points, and could cause problems in practical implementation. Expressing the curvature with respect to the first-order derivative of the network can give information on the magnitude of the curvature coefficients when $N$ is almost constant in some region of the dataspace.}

Notice that, since in our frame $\pa{e_k:=\sum_i\partial_{i} p_k\pa{x}\partial_{i}}$, we have
$\sum_i \partial_{i} \ln p_k(x) \partial_{i} = \frac{1}{p_k(x)} \sum_i{\partial_{i} p_k(x) \partial_{i}} = \frac{1}{p_k(x)} e_k$, we may forget the logarithm in the computations
as it contributes only by a scalar factor.

Let $U = \pa{\sum_k\partial_{k} p_i \partial_{k} p_j}_{i,j}$ be the matrix of the dot products of the partial derivatives of
the probabilities. 

With this notation, if $P = \diag\set{p_k,~ k=1,\ldots,C}$, then 
\[
\hat{D}_{a,b} = \pa{U P^{-1} U}_{a,b}
.\]

Besides, if $J(p)  = \pa{\der{p_a}{x_i}}_{\substack{a=1,\ldots,C \\i=1,\ldots,d}}$ is the Jacobian matrix of first order derivatives, then $U = J(p)J(p)^T$.

Notice that $P$ might not be numerically invertible, whenever some classes are very unlikely and
thus with a probability close to zero. The goal of this section, and the following one, is thus to derive the computations of the curvature forms in a way that is numerically stable. This will allow us to implement the curvature forms computations on a computer.

\medskip
\begin{prop}\label{prop:lie_bracket_nablap_a}
Let the notation be as above. We have:
\begin{equation}
\bra{e_a, e_b} = \Hess{p_b} e_a - \Hess{p_a} e_b 
\end{equation} 
with $\Hess{f} = \pa{\partial_i\partial_j f}_{i,j=1,\ldots,d}$ the matrix of second order partial derivatives.
\end{prop}

\begin{proof}
First, recall that $\bra{u^k\partial_k, v^l \partial_l} = \pa{u^k \der{v^l}{x^k} - v^k \der{u^l}{x^k} } \partial_l$. Thus, with $u^k = \der{p_a}{x^k}$ and $v^l = \der{p_b}{x^l}$ we get:
\begin{align*}
\bra{e_a, e_b} &= \pa{\sum_k \der{p_a}{x^k} \der{^2 p_b}{x^k\partial x^l} -
\sum_k\der{p_b}{x^k} \der{^2 p_a}{x^k\partial x^l} } \partial_l \\
&= \Hess{p_b} e_a - \Hess{p_a} e_b 
.\end{align*}
\end{proof}

\begin{prop}
If $s$ represent the score vector, i.e. the output of the neural network before going through
the softmax function then, for all $j=1,\ldots,C$,
\begin{equation}\label{eq:nabla_p_a}
\partial_{j} p_i = \sum_k p_i \pa{\delta_{ik} - p_k} \partial_{j} s_k 
\quad \text{and} \quad 
\partial_{j} \ln p_i = \sum_k \pa{\delta_{ik} - p_k} \partial_{j} s_k 
.\end{equation} 
In term of Jacobian matrices, this rewrites as
\[
J\pa{p} = \pa{P - pp^T} J\pa{s}
.\]
\end{prop}

\begin{proof}
This is simply due to the fact that $p = \mathrm{softmax}\pa{s}$ and that the derivative of the
softmax function is
$\der{\textrm{softmax}(x)_i}{x_k} = \textrm{softmax}(x)_i \pa{\delta_{i,k} - \textrm{softmax}(x)_k} $.
\end{proof}

We now give other formulae for the second order partial derivatives of the probabilities.

\medskip
\begin{prop}
Let the notation be as above. We have:
\begin{equation}\label{eq:H_p_a_simple}
\Hess{p_a}_{ij} =  \frac{\partial_{i} p_a \partial_{j} p_a}{p_a} - p_a \sum_{k=1}^C \partial_{i} p_k \partial_{j} s_k
.\end{equation}
\end{prop}

\begin{proof}
\begin{align*}
\Hess{p_a}_{ij} &= \partial_{i}\pa{\partial_{j} p_a } = \partial_{i}\pa{\sum_k p_a \pa{\delta_{ik} - p_k} \partial_{j} s_k} \\
								&= \partial_{i} p_a \underbrace{\sum_k \pa{\delta_{ak} - p_k} \partial_{j} s_k}_{= \frac{1}{p_a} \partial_{j} p_a} + p_a \sum_k \partial_{i}\pa{\pa{\delta_{ak}-p_k} \partial_{j} s_k} \\
&= \frac{1}{p_a} \partial_{i} p_a \partial_{j}p_a + p_a \sum_k \pa{\delta_{ak} - p_k} \Hess{s_k} -
p_a \sum_k \partial_{i} p_k \partial_{j} s_k
.\end{align*}
But $\Hess{s_k} = 0$ almost everywhere by \ref{H2}.
\end{proof}

\medskip
However, with this form, the probability at the denominator will cause some instability problems,
whenever the network is sufficiently trained. Thus, we express $\Hess{p_a}$ in another form below.

\medskip
\begin{prop}
Let the notation be as above. We have:
\begin{align*}
\Hess{p_a}_{ij} &= \sum_k \bra{\pa{\delta_{ak}-p_k} \partial_{i} p_a - p_a \partial_{i} p_k}  \partial_{j} s_k \\
&= p_a \sum_k \bra{\pa{\delta_{a k} - p_k} \sum_l \pa{\delta_{a l} -p_l}\partial_{i} s_l - \partial_{i} p_k}  \partial_{j} s_k
.\end{align*}
\end{prop}

\begin{proof}
The proof is a straightforward combination of~\autoref{eq:nabla_p_a} and \autoref{eq:H_p_a_simple}.
We simply need to replace $\partial_{i} p_a$ twice in the expression~\ref{eq:H_p_a_simple} by~\ref{eq:nabla_p_a}.
\end{proof}

\begin{lemma}
    \begin{equation}
		e_a \pa{D}(x) = J(s)^T A_a J(s)
    \end{equation}
    with 
    \[
        \pa{A_a}_{kl} = \sum_i \partial_{i} p_a \pa{\pa{\delta_{kl}-p_l}\partial_{i} p_k - p_k \partial_{i} p_l}
    \]
\end{lemma}

\begin{proof}
\begin{align*}
				e_a \pa{D}(x) &= \sum_i \partial_i p_a \partial_i D(x) \\
                  &= \sum_i \partial_i p_a \partial_i \pa{J(s)^T \pa{P-pp^T}J(s)} \\
                  &= \sum_i \partial_i p_a J(s)^T \partial_i\pa{P-pp^T} J(s) \quad \text{because } \Hess{s_k} = 0 
    .\end{align*}
    Indeed, $\Hess{s_k} = 0$ for all $k$ almost everywhere by \ref{H2}. Then at the indexes $k,l$, we get:
    \begin{align*}
        \partial_i \pa{P-pp^T}_{kl} &= \partial_i \pa{p_k \delta_{kl} - p_k p_l} \\
                      &= \delta_{kl}\partial_i p_k - p_l \partial_i p_k - p_k \partial_i p_l
    .\end{align*}
    Thus, by multiplying with $\partial_i p_a$ and summing over $i$, it gives the expression of $A_a$.
\end{proof}

\begin{prop}\label{prop:e_g_e_e}
    Recall that the neural network satisfies \ref{H2}, then
\begin{align*}
    e_a \pa{g^D\pa{ e_b, e_c }} &= M_{a,c,b} + M_{a,b,c} + \pa{J(s) e_b}^T A_a J(s) e_c   
    \intertext{with}
    M_{a,b,c} &= \pa{\Hess{p_b} e_a}^T J(s)^T \pa{P - pp^T} J(s)e_c
.\end{align*}
\end{prop}

\begin{proof}
    The proof is straightforward.
    \begin{align*}
        e_a\pa{g^D\pa{e_b, e_c}} &= \sum_i \partial_i p_a \partial_i \pa{e_b^T D(x) e_c} \\
                            &= \sum_i \partial_i p_a \partial_i \pa{e_b^T }D(x) e_c + \sum_i \partial_i p_a e_b^T \partial_i \pa{D(x)} e_c \\
                            & \qquad+ \sum_i \partial_i p_a e_b^T D(x) \partial_i e_c \\
                            &= e_a^T \Hess{p_b}^T D(x) e_c + e_b^T e_a \pa{D(x)} e_c + e_b^T D(x) \Hess{p_c} e_a \\
                            &= M_{a,b,c} + e_b^T J(s)^T A_a J(s) e_c + \pa{e_a^T \Hess{p_c}^T D(x) e_b}^T\\ 
                            &= M_{a,b,c} + e_b^T J(s)^T A_a J(s) e_c + \pa{M_{a,c,b}}^T 
    .\end{align*}
    And $M_{a,c,b}$ is a scalar, thus $M_{a,c,b}^T = M_{a,c,b}$, hence the result.
\end{proof}

Then, we use the following proposition to remove the second order derivative in $M$. An alternative formula for this expression is given in~\ref{prop:derivationscalarproduct}. This facilitates the computations with automatic differentiation methods.

\begin{prop}
    \begin{equation}
        \Hess{p_a} e_b = \pa{\sum_{k,i} \pa{\delta_{ak} - p_k}\pa{\partial_{i} s_k \partial_{i} p_b}} e_a - \sum_{k,i} p_a \pa{\partial_{i} s_k \partial_{i} p_b} e_k
    .\end{equation}
\end{prop}

\begin{proof}
    \begin{align*}
		\Hess{p_a} e_b &= \sum_i \pa{ \sum_k \pa{\pa{\delta_{ak} - p_k}\partial_{j} p_a  - p_a \partial_{j} p_k } \partial_{i} s_k } \partial_{i} p_b \\
								 &= \sum_i \pa{ \sum_k \pa{\delta_{ak} - p_k}\partial_{j} p_a \partial_{i} s_k} \partial_{i} p_b - p_a \sum_{i}\pa{\sum_k \partial_{j} p_k \partial_{i} s_k} \partial_{i} p_b \\
                   &=  \pa{\sum_{k,i} \pa{\delta_{ak} - p_k}\partial_{i} s_k \partial_{i} p_b} \partial_{j} p_a  - p_a \sum_{k,i}  \pa{\partial_{i} s_k \partial_{i} p_b} \partial_{j} p_k 
    .\end{align*}
\end{proof}

\begin{prop}
				\begin{equation}
								g^D\pa{e_a, \bra{e_b , e_c}} = \sum_{i,j,k}\partial_{i} p_a D(x)_{ij} \pa{\partial_{j}\partial_{k}{p_c}\partial_{k} p_b - \partial_{j}\partial_{k}{p_b}\partial_{k} p_c}
				.\end{equation}

\end{prop}

\begin{proof}
	Straightforward from the previous propositions and lemmas.
\end{proof}

With all these propositions, the connection forms can be computed directly without numerical instabilities.  

\section{Computation of the curvature forms}
\label{sec:computation-curv-forms}

In this section we conclude our calculation of the curvature forms.
In Prop. \ref{curv-form} we wrote the explicit expression for the curvature form as:
\begin{equation}\label{eq-conn-form}
\Omega \pa{X,Y}  = \pa{d\omega} \pa{X,Y} + \omega\wedge\omega\pa{X,Y}
.\end{equation}
where $\omega$ denotes the (Levi-Civita) connection form. To ease the reading we go back to notation of Section~\ref{sec:cartan-moving-frames} and we set:
\[
e_k:=\sum_i\partial_{i} p_k\pa{x} \partial_{i}, \quad \hbox{for} \quad k=1, \dots, C-1%
.\]
We then can express explicitly the connection form as:
\[
				\nabla_X e_j = \sum \omega_j^i (X) e_i
.\]
The wedge product of the connection forms in (\ref{eq-conn-form}) can thus be easily computed with the propositions
of the previous section, because of formula (\ref{for1}) report here: 
\beq\label{f}
\pa{\omega \wedge \omega}\pa{X,Y} = \omega(X)\omega(Y) - \omega(Y)\omega(X)
.\eeq 

The exterior derivative of $\omega$ remains to be computed and it is more complicated. We recall the formula (\ref{for2}):
\beq
\pa{d\omega}(X,Y) = X\omega(Y) - Y\omega(X) - \omega \pa{\bra{X,Y}}
.\eeq

The last term is computed via the following proposition.

\medskip
\begin{prop}\label{prop:omega_bra}
Let the notation be as above. Then:
    \begin{align*}
        \omega_j^i\pa{\bra{e_a,e_b}} = & \pa{\sum_{k,l} \pa{\delta_{bk} - p_k}\pa{\partial_{l} s_k \partial_{l} p_a}} \omega_j^i \pa{e_b}\\
                        & - \pa{\sum_{k,l} \pa{\delta_{ak} - p_k}\pa{\partial_{l} s_k \partial_{l} p_b}} \omega_j^i \pa{e_a}\\ 
                        & - \sum_{k,l} \pa{\partial_{l} s_k \pa{ p_b \partial_{l} p_a - p_a \partial_{l} p_b}} \omega_j^i \pa{e_k}
    .\end{align*}
\end{prop}

\begin{proof}
    \begin{align*}
        \omega_j^i \pa{\bra{e_a, e_b}} &= \omega_j^i \pa{\Hess{p_b}e_a} - \omega_j^i \pa{\Hess{p_a}e_b}
    .\end{align*}
    Besides,
    \begin{align*}
        \omega_j^i\pa{\Hess{p_a}e_b} &= \omega_j^i \pa{\pa{\sum_{k,l} \pa{\delta_{ak} - p_k}\pa{\partial_{l} s_k \partial_{l} p_b}} e_a - \sum_{k,l} p_a \pa{\partial_{l} s_k \partial_{l} p_b} e_k
} \\
                          &= \pa{\sum_{k,l} \pa{\delta_{ak} - p_k}\pa{\partial_{l} s_k \partial_{l} p_b}} \omega_j^i \pa{e_a} - \sum_{k,l} p_a \pa{\partial_{l} s_k \partial_{l} p_b} \omega_j^i \pa{e_k}
    .\end{align*}
    Thus,
    \begin{align*}
        \omega_j^i\pa{\bra{e_a,e_b}} = & \pa{\sum_{k,l} \pa{\delta_{bk} - p_k}\pa{\partial_{l} s_k \partial_{l} p_a}} \omega_j^i \pa{e_b}\\
                        & - \pa{\sum_{k,l} \pa{\delta_{ak} - p_k}\pa{\partial_{l} s_k \partial_{l} p_b}} \omega_j^i \pa{e_a}\\ 
                        & - \sum_{k,l} \pa{\partial_{l} s_k \pa{ p_b \partial_{l} p_a - p_a \partial_{l} p_b}} \omega_j^i \pa{e_k}
    .\end{align*}
\end{proof}

We now tackle the question of determining the first two terms in (\ref{for2}).

\medskip
\begin{observation}\label{prop:e_omega_e}
We notice that to compute $e_a \pa{\omega^i_j \pa{e_b}}$, we can use the fact that: 
\begin{align*}
    e_a \pa{g^D\pa{\nabla_{e_b} e_c,e_d}} &= e_a \pa{\sum_i \omega_c^i \pa{e_b} g^D\pa{e_i, e_d}} \\
                      &= \sum_i e_a \pa{\omega_c^i\pa{e_b}}g^D\pa{e_i,e_d} + \sum_i \omega_c^i\pa{e_b} e_a\pa{g^D\pa{e_i, e_d}} \\
    \iff \sum_i \hat{D}_{d,i} e_a \pa{\omega_c^i \pa{e_b}} &= \underbrace{e_a \pa{g^D\pa{\nabla_{e_b} e_c, e_d}} - \sum_i \omega_c^i\pa{e_b}e_a\pa{g^D\pa{e_i,e_d}}}_{N_{a,b,c,d}} \\
    \iff  e_a \pa{\omega_c^\cdot\pa{e_b}} &= \hat{D}^{-1} N_{a,b,c,\cdot}
.\end{align*}
Hence we need to compute:
$e_a g^D\pa{\nabla_{e_b} e_c, e_d}$.
\end{observation}



\begin{theorem}
The expression of the curvature is given by:
\begin{multline}\label{thm-conn-form}
\Omega^i_j \pa{e_a,e_b}={\pa{\hat{D}^{-1}N_{a,b,j,\cdot}}_i -\pa{\hat{D}^{-1}N_{b,a,j,\cdot}}_i - \omega^i_j \pa{\bra{e_a,e_b}} } \\
 + \sum_k \bra{\omega^i_k\pa{e_a}\omega^k_j\pa{e_b} - \omega^i_k\pa{e_b}\omega^k_j\pa{e_a}}
\end{multline}


where $N_{a,b,c,\cdot} = \pa{N_{a,b,c,d}}_{d=1,\ldots,C-1}$ is the vector defined in Observation~\ref{prop:e_omega_e} by $e_a \pa{g^D\pa{\nabla_{e_b} e_c, e_d}} - \sum_i \omega_c^i\pa{e_b}e_a\pa{g^D\pa{e_i,e_d}}$ for $d=1,\ldots,C-1$, and  where $\hat{D}=\pa{g^D\pa{e_i,e_j}}_{i,j = 1,\ldots,C-1}$ is the matrix of pairwise metric products of the frame.
\end{theorem}

\medskip
We report the lemmas needed to compute the tensor $N_{a,b,c,d}$ in Appendix~\ref{app:proof-computation-curv-forms}.

\revision{
\begin{remark}
    The practical computation of the curvature forms involves some matrix multiplication, inversion and the computation of the first order derivatives of the network which can be done with automatic differentiation \cite{baydin2018AutomaticDiff}. The most time consuming operation is the differentiation of the network as it scales with the number of layers (thus the complexity of the task to learn), with the input dimension and with the number of classes. Improving the efficiency of this computation is left as future work. 
\end{remark}
}

\section{Experiments} \label{sec:exp}

\eliot{
To understand why the frame $e_a=\sum_i \partial_i p_a \partial_i$ and the DIM were chosen to compute the connection and curvature forms, we shall focus on experiments on the MNIST and the CIFAR10 datasets.
As we will see, this frame and the DIM can provide some explanations to the response of the neural network.\footnote{The code used to produce these results is available at \url{https://github.com/eliot-tron/curvcomputenn/}.}

\revision{Note that in this section, we focus on the DIM and the frame $\pa{e_a}_{a=1}^{C-1}$. We do not inspect the connection or curvature forms and leave them for future work.}

The MNIST dataset is composed of 60k train images of shape $28\times 28$ pixels depicting handwritten digits between $0$ and $9$, and the CIFAR10 dataset is composed of 60k RGB train images of shape $32\times 32$
classified in 10 different classes (see Table~\ref{tab:CIFAR10_classes}).

\input{Table_1}

On Figure~\ref{fig:MNIST-ReLU-DIM} are represented various input points $x$ from the MNIST training set and the corresponding matrix\footnote{We plot here the matrix $\hat{D}$ with $a,b$ going up to $C$, and not $C-1$, to represent all the classes and have an easier interpretation.} $\hat{D}_x=\pa{g_x\pa{e_a,e_b}}_{a,b=1,\ldots,C}$. The neural network used in this experiment is defined as stated in Table~\ref{tab:MNIST_architecture}. This network has been trained on the test set of MNIST with stochastic gradient descent until convergence (98\% accuracy on the test set).

On Figure~\ref{fig:MNIST-ReLU-DIM}, it can be seen that the matrices $\hat{D}_x$ have only a few main components indicating which probabilities are the easiest to change. A large (positive) component on the diagonal at index $i$ suggests that one can increase or decrease easily $p_i(x)$ by moving in the directions $\pm e_i$. A negative (resp. positive) component at position $(i,j)$ indicates that, starting from the image $x$, classes $i$ and $j$ are in opposite (resp. the same) directions: increasing $p_i(x)$ will most likely decrease (resp. increase) $p_j(x)$.

For instance, the first image on the top left is correctly classified as a 2 by the network, but since the coefficient $(3,3)$ of matrix $\hat{D}$ is positive too, it indicates that class 3 should be easily reachable. This makes sense, as the picture can also be interpreted as a part of the 3 digit. Negative coefficients at $(3,2)$ and $(2,3)$ thus indicates that going in the direction of a 3 will decrease the probability of predicting a 2. On the second picture, the same phenomenon arises but with classes 2 and 8. Indeed, the buckle in the bottom part of the picture brings it closer to an 8.

\begin{remark}
Be careful as teal colored coefficients on Figure~\ref{fig:MNIST-ReLU-DIM} and Figure~\ref{fig:CIFAR10-ReLU-DIM} are not exactly zero but rather very low compared to the few main ones that are yellow and purple.
\end{remark}

\input{Table_2}

\begin{figure}[ht]
    \centering
    \includegraphics[width=0.8\textwidth]{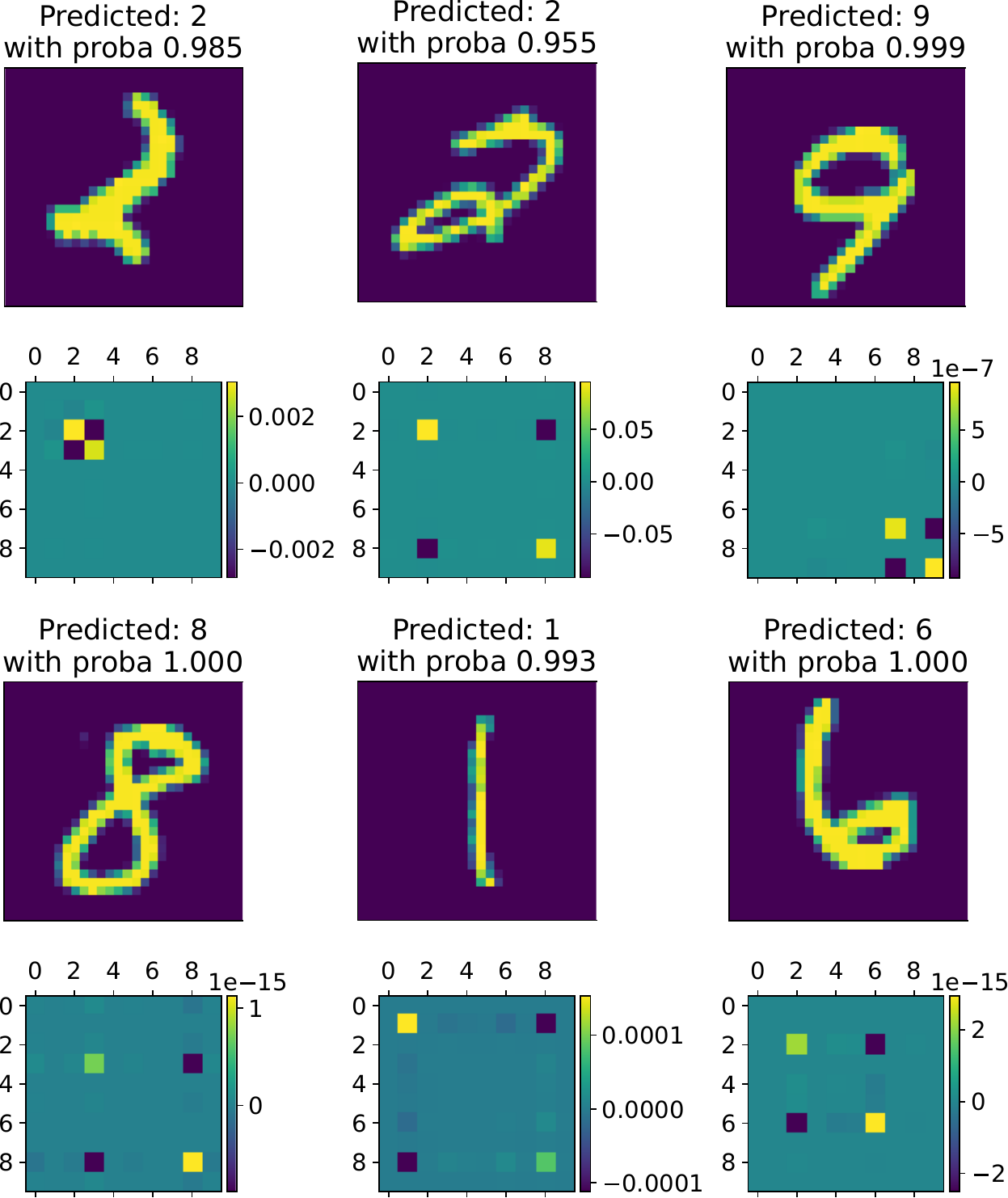}
    \caption{Couples of input point $x$ (above) and the corresponding matrix $\hat{D}(x) = \pa{g^D_x\pa{e_a,e_b}}_{a,b=1,\ldots,C}$ (below) on MNIST.}
    \label{fig:MNIST-ReLU-DIM}
\end{figure}

On Figure~\ref{fig:CIFAR10-ReLU-DIM} are represented various input points $x$ from the CIFAR10 training set and the corresponding matrix $\hat{D}_x=\pa{g_x\pa{e_a,e_b}}_{a,b=1,\ldots,C}$. The neural network used in this experiment is defined as stated in Table~\ref{tab:CIFAR10_architecture}. This network has been trained on the test set of CIFAR10 with stochastic gradient descent until convergence (84\% accuracy on the test set).

On Figure~\ref{fig:CIFAR10-ReLU-DIM}, it can be seen that the matrices $\hat{D}_x$ have only a few main components indicating which probabilities are the easiest to change, i.e. with a similar behavior as seen above for MNIST. The interpretation of matrix $\hat{D}$ then is identical.

For instance, the first image on the top left is correctly classified as a dog (class No. 5) by the network, but since the coefficient $(3,3)$ of matrix $\hat{D}$ is positive too, it indicates that class No. 3 ``cat'' should be easily reachable. This makes sense since dogs and cats form a subclass of similar little animals. Negative coefficients at $(3,2)$ and $(2,3)$ thus indicates that going in the direction of the cat will decrease the probability of predicting a dog. There are also positive coefficients at $(3,7)$ and $(7,3)$ indicating that going in the direction of the cat should also slightly increase the probability of seeing a horse. Again, this makes sense as they all belong in the animal subclass.

On the second picture, the dog can be transformed into a cat or a frog, probably because of the green cloth around its neck.

\input{Table_3}
}

\begin{figure}[ht]
    \centering
    \includegraphics[width=0.8\textwidth]{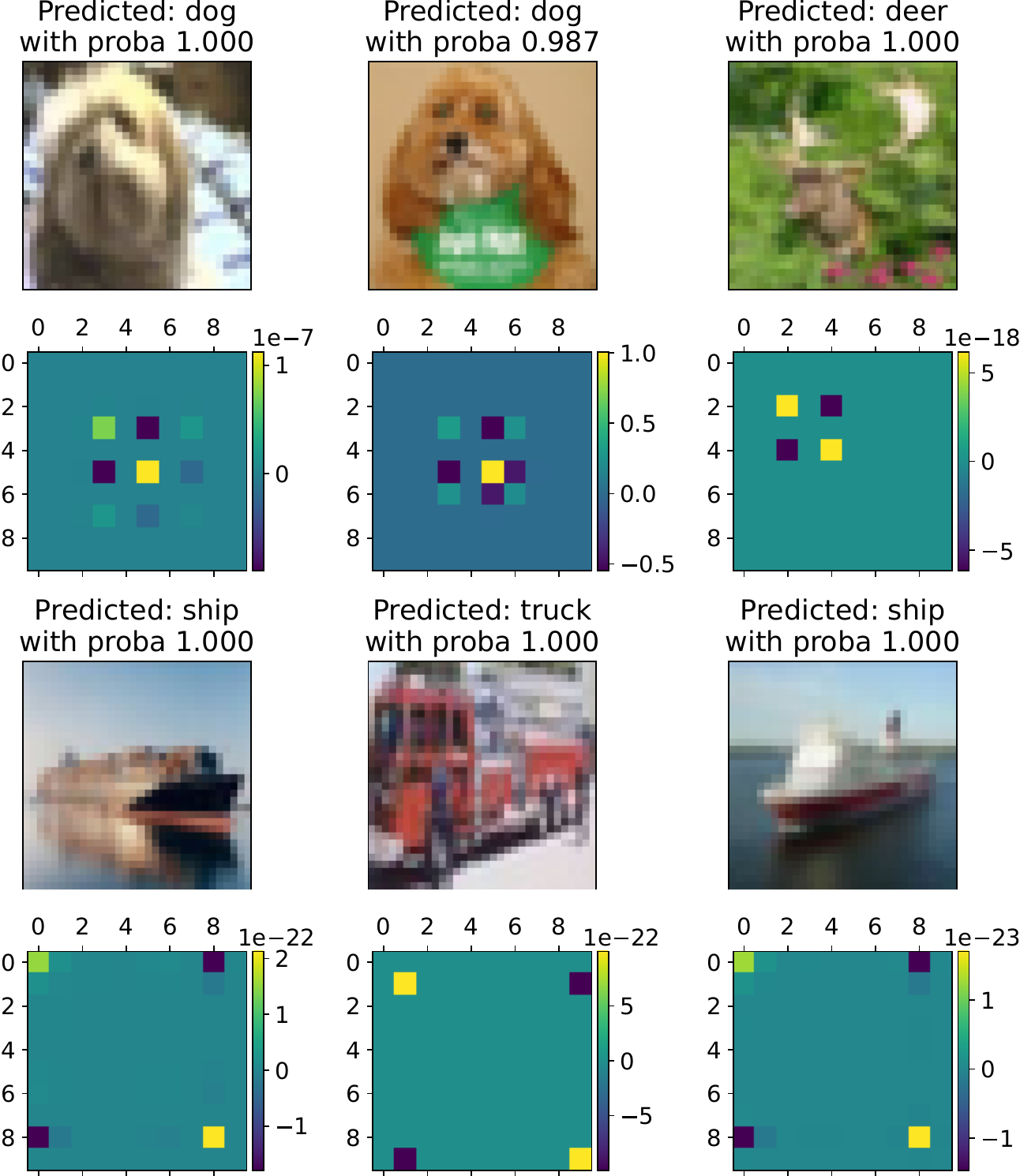}
    \caption{Couples of input point $x$ (above) and the corresponding matrix $\hat{D}(x) = \pa{g^D_x\pa{e_a,e_b}}_{a,b=1,\ldots,C}$ (below) on CIFAR10.}
    \label{fig:CIFAR10-ReLU-DIM}
\end{figure}

\medskip

\revision{
More experiments can be found in Appendix~\ref{app:more_exp} with full details to ensure reproducibility.
}

\section{Conclusion}
\label{sec:conclusion}
In this study, we have shown that analyzing the geometry of the data manifold of a neural network using Cartan moving frames is natural and practical. Indeed, from the Fisher Information Matrix, we define a corresponding Data Information Matrix that generates a foliation structure in the data space. The partial derivatives of the probabilities learned by the neural network define a Cartan moving frame on the leaves of the data manifold. We detail how the moving frame can be used as a tool to provide some explanations on the classification changes that can easily happen or not around a given data point. Experiments on the MNIST and CIFAR datasets confirm the relevance of the explanations provided by the Cartan moving frames and the corresponding Date Information Matrix. For very large neural network, the theory still holds. However, the method might be limited by the computational requirements of the partial derivatives calculation for each class (usually obtained through automatic differentiation).
We believe that combining the moving frame, the connection and the curvature forms should provide new insights to build more advanced explainable AI tools. This is work in progress.


\newpage

\section*{Data availability} Data sets used in this article can be found online at:
\begin{itemize}
    \item MNIST dataset \cite{lecun1998mnist}: \url{https://yann.lecun.com/exdb/mnist/}
    \item CIFAR10 dataset \cite{cifar-10}: \url{https://www.cs.toronto.edu/~kriz/cifar.html}
    \item Code: \url{https://github.com/eliot-tron/curvcomputenn/}
\end{itemize}
We used Python and PyTorch to access these data.

\bibliography{main}

\appendix
\revision{
\section{Cartan moving frames}\label{app:cartan_moving_frames}

In this appendix we briefly review the notion of moving frames for differentiable
manifolds, highlighting the key geometric properties, with some pictures
to help the intuition. For more details see \cite{tu2017differential} and \cite{petersen}.

\medskip
The theory of Cartan moving frames ({\sl repère mobile}) stems from the
Frenet– Serret frame theory, which lead to the celebrated Frenet-Serret
formulas, developed independently by Frenet and Serret,
in the XIX century, to describe the motion of a particle
on a curve in $\R^3$.  In fact, given smooth curve $\gamma$ in $\R^3$, a Frenet-Serret frame is
an orthonormal basis of $\R^3$ attached at each point of
$\gamma$ and consisting of the (normalized) tangent, normal and conormal
vectors (see Fig. \ref{fig:frenet}), obtained via subsequent derivations
of the curve $\gamma$.

\begin{figure}[h]
    \centering
    \includegraphics[width=0.8\textwidth]{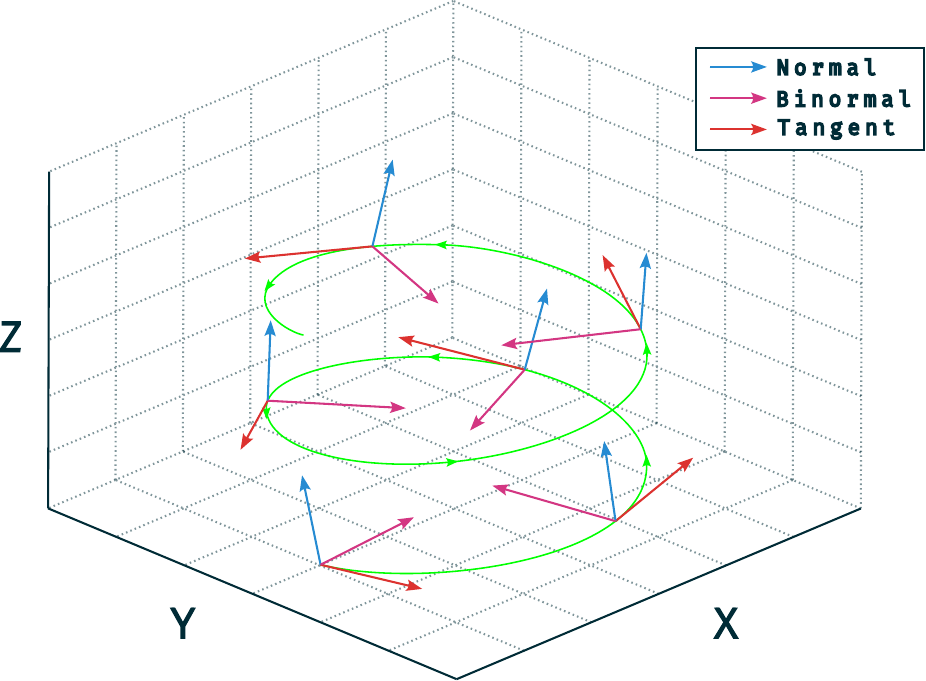}
    \caption{Frenet-Serret Frames}
    \label{fig:frenet}
\end{figure}

Notice that, once we give the curve $\gamma$, the definition of
its Frenet-Serret frame is independent from the coordinates of $\R^3$; 
it is in fact an ``intrinsic'' reference frame for $\R^3$, associated to
each point of the curve.

\medskip
The notion of moving frames is a far reaching generalization of Frenet-Serret frames, where
we replace the smooth curve $\gamma$, with a smooth manifold $M$.

\medskip
\begin{definition}
We define
a \textit{moving frame} on a differentiable manifold $M$ a collection of vector fields
$e_1, e_2, \dots, e_n$, which are a basis of $T_pM$, the tangent space at each
point $p \in M$.
(see Fig. \ref{fig:moving}). We also ask such collection to be varying smoothly with the
point $p$ in $M$. The assignment of a moving frame at each point $p \in M$ is called
a \textit{soldering form} $\theta: T_pM \lra \R^n$.
\end{definition}

Hence, once a soldering form
is given, we have an intrinsic point of view on the geometry of a manifold and we can,
in some sense, effectively replace the notion of local coordinates. Such
point of view was
successfully exploited to describe, in
a local coordinate free manner, some physical theories, most
notably  {\sl general relativity}, see
\cite{carroll}.

In Riemannian geometry, we take the 
moving frames to be orthonormal frames, that is, at each point $p \in M$ the frame
is an orthonormal basis of $T_pM$, with respect to the metric on the manifold.

\begin{figure}[h]
    \centering
    \includegraphics[width=0.6\textwidth]{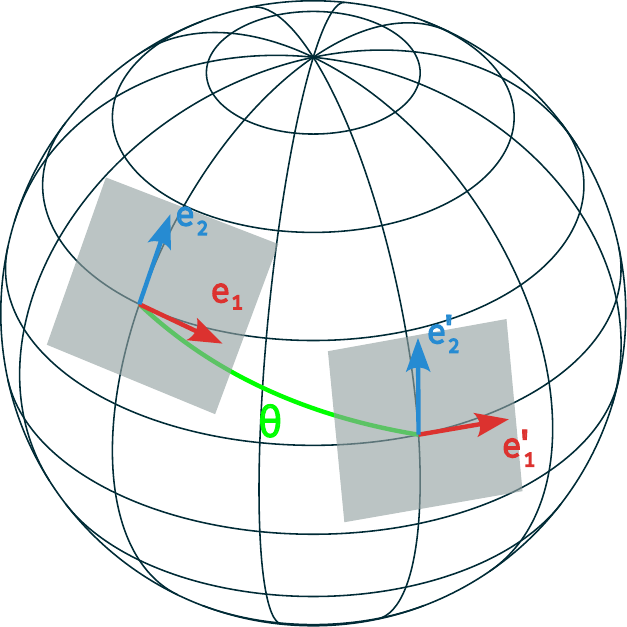}
    \caption{Soldering form on a sphere.}
    \label{fig:moving}
\end{figure}

\medskip
In our application, that is the geometry of the data manifold $\cL$, such moving frame comes
naturally from the neural network and it is given by the gradients of the probability
output, so that we immediately have a soldering form on $\cL$ as follows:
\begin{equation}\label{soldering}
\theta:T_p\cL \lra \R^{C-1}, \qquad e_i \mapsto \nabla p_i
\end{equation}
Moreover we have a natural Riemannian metric on the data leaf via the DIM
(data information matrix). 
The soldering form defines uniquely (Prop. 1.1 pg 378 in \cite{petersen}) a matrix
of one forms $\omega=(\omega^i_j$), called (evocatively) {\it spin connection}.
So we can actually view the formula in Def. \ref{def:connection_forms} in two
equivalent ways:
\begin{enumerate}
\item Given our soldering form $\theta$ on the Riemannian manifold $\cL$ as in (\ref{soldering}),
we have a spin connection $\omega$ which gives the Levi Civita connection by the
formula (see (\ref{lc-con}) in the main text):
$$
	\nabla_X e_j = \sum_i \omega_j^i (X) e_i
$$
\item Given the metric via the DIM on the data leaf, we can compute its
(unique) Levi Civita connection and this will give us $\omega$ the spin connection, via
the above formula.
\end{enumerate}

We choose to take the second point of view in our paper, to avoid introducing too much
terminology and results in moving frames theory, since we prefer to rely on the 
correspondence between Levi-Civita connections and metrics, possibly more familiar to
most readers.

\medskip
These equivalent ways to view the Levi Civita connection is an illustration of the power
of the moving frames method as a coordinate independent way to approach
Riemannian geometry. We can indeed compute the curvature form as:
\begin{equation}\label{eq-conn-form-app}
\Omega \pa{X,Y}  = \pa{d\omega} \pa{X,Y} + \omega\wedge\omega\pa{X,Y}.
\end{equation}
as we reported in (\ref{eq-conn-form}) in the main text.
Such curvature form is linked with the usual Riemannian curvature tensor
via the equation:
$$
R(X,Y)e_j=\sum \Omega_j^i(X,Y) e_i
$$
(see \cite{petersen} pg 379).

\medskip
For more details on the method of moving frames and its link with the equivalent
theory of Riemannian geometry making use of local coordinates, we invite
the reader to consult \cite{petersen} App. B.

}

\section{An auxiliary lemma}
\label{app:an-auxiliary-lemma}

We report here, for convenience, an alternative formula for the expression used in \ref{prop:e_g_e_e}.

\begin{prop}\label{prop:derivationscalarproduct}
				Recall that $U = \pa{u_{i,j}} = \pa{\ang{e_i, e_j}_{\re}}_{i,j}$. Then
				\begin{align*}
								e_a \pa{g^D\pa{e_b, e_c}} = \quad &\sum_j u_{aj}u_{bj} u_{cj} \frac{1}{p_j^2} \\
								+ & \pa{\frac{u_{ab}}{p_b} + \frac{u_{ac}}{p_c}}\sum_j \frac{1}{p_j} u_{bj}u_{cj}\\
								- & \sum_j \pa{u_{cj} p_b + u_{bj}p_c} \frac{1}{p_j} \sum_{k,i} \pa{\partial_{i} p_a \partial_{i} s_k} u_{jk} \\
								- &\sum_i \pa{\partial_{i} p_a \partial_{i} s_k} \pa{u_{cj} u_{bk} + u_{bj} u_{ck}}
				.\end{align*}
				Notice that on the second line, we recognise the expression of $g^D\pa{e_b, e_c}$ for the sum.
\end{prop}

\begin{proof}
				\begin{align*}
								e_a \pa{g^D\pa{e_b, e_c}} &= \sum_i \partial_i p_a \partial_i \pa{g^D\pa{e_b, e_c}} \\
																													&= \sum_i \partial_i p_a \partial_i \pa{\sum_j \frac{1}{p_j} u_{bj} u_{cj}} \\
																													&= \sum_{i,j} \partial_i p_a \pa{\pa{\partial_i\frac{1}{p_j}}u_{cj}u_{bj} + \frac{1}{p_j}\pa{\partial_i u_{bj}}u_{cj} + \frac{1}{p_j} u_{bj} \pa{\partial_i u_{cj}}}
				.\end{align*}

				Now we compute $\partial_i u_{bj}$ :
				\begin{align*}
								\partial_i u_{bj} &= \partial_i \pa{\sum_l\partial_l p_b \partial_l p_j} \\
																	&= \sum_l\underbrace{\pa{\partial_i \partial_l p_b}}_{=\Hess{p_b}_{i,l}} \partial_l p_j + \sum_l \partial_l p_b \pa{\partial_i \partial_l p_j} \\
																	&= \sum_l \pa{\frac{\partial_i p_b \partial_l p_b}{p_b} - p_b \sum_k \partial_i s_k \partial_l p_k} \partial_l p_j  \\
																	& \quad + \sum_l\partial_l p_b \pa{\frac{\partial_i p_j \partial_l p_j}{p_j} - p_j \sum_k \partial_i s_k \partial_l p_k} \\
																	&= \sum_l\pa{\frac{\partial_i p_b}{p_b} + \frac{\partial_i p_j}{p_j}} u_{aj} - \sum_l\pa{p_a \partial_l p_j + p_j \partial_l p_a} \sum_k \partial_i s_k \partial_l p_k
				.\end{align*}

				Thus

				\begin{align*}
								e_a \pa{g^D\pa{e_b, e_c}} = \quad & \sum_j u_{bj} u_{cj} u_{aj} \frac{-1}{p_j^2} \\
								 + & \sum_j \pa{\frac{u_{ab}}{p_b} + \frac{u_{aj}}{p_j}} \frac{1}{p_j} u_{bj} u_{cj}  \\
								 - &\sum_j \frac{1}{p_j} u_{cj} \pa{p_b \sum_{k,l} \pa{\partial_l p_a \partial_l s_k}u_{jk} + p_j \sum_{k,l} \pa{\partial_l p_a \partial_l s_k} u_{bk}} \\ 
								 + & \sum_j \pa{\frac{u_{ac}}{p_c} + \frac{u_{aj}}{p_j}} \frac{1}{p_j} u_{bj} u_{cj}  \\
								 - &\sum_j \frac{1}{p_j} u_{bj} \pa{p_c \sum_{k,l} \pa{\partial_l p_a \partial_l s_k}u_{jk} + p_j \sum_{k,l} \pa{\partial_l p_a \partial_l s_k} u_{ck}}
				.\end{align*}
\end{proof}

To get a numerically stable expression for $e_a g^D\pa{ e_b, e_c }$, we can develop  $u$ with \autoref{eq:nabla_p_a} and the probability at the denominator will cancel out with the one in factor of \autoref{eq:nabla_p_a}.

\section{Further computations of the curvature forms}
\label{app:proof-computation-curv-forms}

In this section we report, for convenience, some lemmas necessary for the full derivation of the curvature forms calculations. 

\begin{lemma}
    \begin{equation}
        e_a \pa{e_b\pa{g^D\pa{e_c, e_d}}} = e_a\pa{M_{b,d,c}} + e_a\pa{M_{b,c,d}} + e_a\pa{e_c^T J(s)^T A_b J(s) e_d }
    .\end{equation}
\end{lemma}

\begin{proof}
    \begin{align*}
        e_a \pa{e_b\pa{g^D\pa{e_c, e_d}}} &= e_a\pa{M_{b,d,c} + M_{b,c,d} + e_c^T J(s)^T A_b J(s) e_d } \\
                      &= e_a\pa{M_{b,d,c}} + e_a\pa{M_{b,c,d}} + e_a\pa{e_c^T J(s)^T A_b J(s) e_d }
    .\end{align*}

\end{proof}

\begin{lemma}
    \begin{multline}
        e_a \pa{M_{b,c,d}} = ~ e_a\pa{\Hess{p_c}e_b}^T D(x) e_d \\
                    + \pa{\Hess{p_c}e_b}^T J(s)^T A_a J(s) e_d \\
                    + \pa{\Hess{p_c}e_b}^T D(x) \pa{\Hess{p_d} e_a}
    .\end{multline}
\end{lemma}

\begin{proof}
    \begin{align*}
        e_a \pa{M_{b,c,d}} = &~ e_a \pa{e_b^T \Hess{p_c}^T D(x) e_d} \\
                   = &~ e_a\pa{\Hess{p_c}e_b}^T D(x) e_d \\
                     & + \pa{\Hess{p_c}e_b}^T e_a\pa{D(x)} e_d \\
                   & + \pa{\Hess{p_c}e_b}^T D(x) e_a\pa{e_d} \\
                   = &~ e_a\pa{\Hess{p_c}e_b}^T D(x) e_d \\
                   & + \pa{\Hess{p_c}e_b}^T J(s)^T A_a J(s) e_d \\
                   & + \pa{\Hess{p_c}e_b}^T D(x) \pa{\Hess{p_d} e_a}
    .\end{align*}
\end{proof}

\begin{lemma}
    \begin{align*}
        e_a \pa{\Hess{p_c} e_b}
        = &~ \sum_k \Big[- \pa{u_{ak}}\pa{\sum_i \partial_i s_k \partial_i p_b} e_c \\
         &\qquad + \pa{\delta_{ck} - p_k} \pa{\sum_{i,j} \partial_i s_k \partial_i\partial_j p_b\partial_j p_a}e_c \\
         &\qquad + \pa{\delta_{ck} - p_k} \pa{\sum_i \partial_i s_k \partial_i p_b}  \Hess{p_c}e_a \Big] \\
         &~ - \sum_k \Big[  \pa{u_{ac}}\pa{\sum_i \partial_i s_k \partial_i p_b} e_k\\
          &\qquad + p_c \pa{\sum_{i,j} \partial_i s_k \partial_i\partial_j p_b\partial_j p_a} e_k\\
         &\qquad + p_c \pa{\sum_i \partial_i s_k \partial_i p_b}\Hess{p_k} e_a
         \Big]
    .\end{align*}
\end{lemma}

\begin{proof}
    The proof is straightforward by developing the following:
    \begin{align*}
        e_a \pa{\Hess{p_c} e_b} = &~ e_a\pa{\pa{\sum_k \pa{\delta_{ck} - p_k}\pa{\sum_i \partial_i s_k \partial_i p_b}} e_c - \sum_k p_c \pa{\sum_i \partial_i s_k \partial_i p_b} e_k} 
    .\end{align*}
\end{proof}

\begin{lemma}
    \begin{align*}
        e_a \pa{e_c^T J(s)^T A_b J(s) e_d} = &~ e_a^T \Hess{p_c}^T J(s)^T A_b J(s) e_d \\
                                  &~ + e_c^T J(s)^T e_a\pa{A_b} J(s) e_d \\
                                  &~ + e_c^T J(s)^T A_b J(s) \Hess{p_d} e_a
    .\end{align*}
\end{lemma}

\begin{lemma}
    \begin{align*}
        e_a \pa{A_b}_{kl} 
                  = &~ e_a^T \Hess{p_b} \pa{\pa{\delta_{kl}-p_l}e_k - p_k e_l} \\ 
                  &~ + \pa{e_a^T e_l} \pa{e_b^T e_k} - p_k e_b^T e_l \\
                  &~ + e_b^T \pa{\pa{\delta_{kl} - p_l} \Hess{p_k} e_a - p_k e_l} \\
                  &~ + {\pa{\delta_{kl} - p_l}\pa{e_b^T  e_k} - \pa{e_a^T e_k} \pa{e_b^T e_l}}
    .\end{align*}
\end{lemma}

\begin{proof}
    The proof is straightforward by developing the following:
    \begin{align*}
        e_a \pa{A_b}_{kl} = &~ \sum_i \partial_i p_a \partial_i \pa{e_b^T \pa{\pa{\delta_{kl}-p_l}e_k - p_k e_l}} 
    .\end{align*}
\end{proof}

\begin{lemma}
    \begin{equation}
    \begin{split}
        e_a \pa{g^D\pa{e_b,\bra{e_c,e_d}}} 
        = &~ e_a^T \Hess{p_b} D(x) \pa{ \Hess{p_d}e_c - \Hess{p_c}e_d } \\
         &~ + e_b^T J(s)^T A_a J(s) \pa{\Hess{p_d}e_c - \Hess{p_c} e_d} \\
         &~ + e_b^T D(x) \pa{B_{a,d,c} - B_{a,c,d}}.
    \end{split}
    \end{equation}
    with 
    \begin{align*}
        B_{a,c,d} = &~ \sum_k \Big[- \pa{u_{ak}}\pa{\sum_i \partial_i s_k \partial_i p_d} e_c \\
         &\qquad + \pa{\delta_{ck} - p_k} \pa{\sum_{i,j} \partial_i s_k \partial_i\partial_j p_d\partial_j p_a}e_c \\
         &\qquad + \pa{\delta_{ck} - p_k} \pa{\sum_i \partial_i s_k \partial_i p_d} \Hess{p_c} e_a\\
         &\qquad - \pa{u_{ac}}\pa{\sum_i \partial_i s_k \partial_i p_d} e_k\\
          &\qquad - p_c \pa{\sum_{i,j} \partial_i s_k \partial_i\partial_j p_d\partial_j p_a} e_k\\
         &\qquad - p_c \pa{\sum_i \partial_i s_k \partial_i p_d}\Hess{p_k} e_a
         \Big]
    \end{align*}
\end{lemma}

\begin{proof}[Proof of the lemma]
    \begin{align*}
        e_a \pa{g^D\pa{e_b,\bra{e_c,e_d}}} = &~ e_a \pa{ e_b^T D(x) \pa{\Hess{p_d}e_c - \Hess{p_c}e_d} } \\
        = &~ e_a^T \Hess{p_b} D(x) \pa{ \Hess{p_d}e_c - \Hess{p_c}e_d } \\
         &~ + e_b^T J(s)^T A_a J(s) \pa{\Hess{p_d}e_c - \Hess{p_c} e_d} \\
         &~ + e_b^T D(x) {e_a\pa{\Hess{p_d}e_c - \Hess{p_c} e_d}}.
    \end{align*}
    Besides, $e_a \pa{\Hess{p_d}e_c}$ has already been computed previously.
\end{proof}

\revision{
\section{More experiments}
\label{app:more_exp}
In this section, we report some more experiments to ensure reproducibility. The code used to compute the DIM is available at \url{https://github.com/eliot-tron/curvcomputenn/}. The seed used for the following experiments is 42. On Table~\ref{tab:MNIST-experiment-seed-42} and Table~\ref{tab:CIFAR10-experiment-seed-42}, we report the extreme values of the matrix $\hat{D}(x) = \pa{g^D_x\pa{e_a,e_b}}_{a,b=1,\ldots,C}$ evaluated on six points of the test set of MNIST (see Figure~\ref{fig:figure5}) and CIFAR10 (see Figure~\ref{fig:figure6}) respectively.

\input{Table_4}

\begin{figure}
    \centering
    \includegraphics[width=0.8\textwidth]{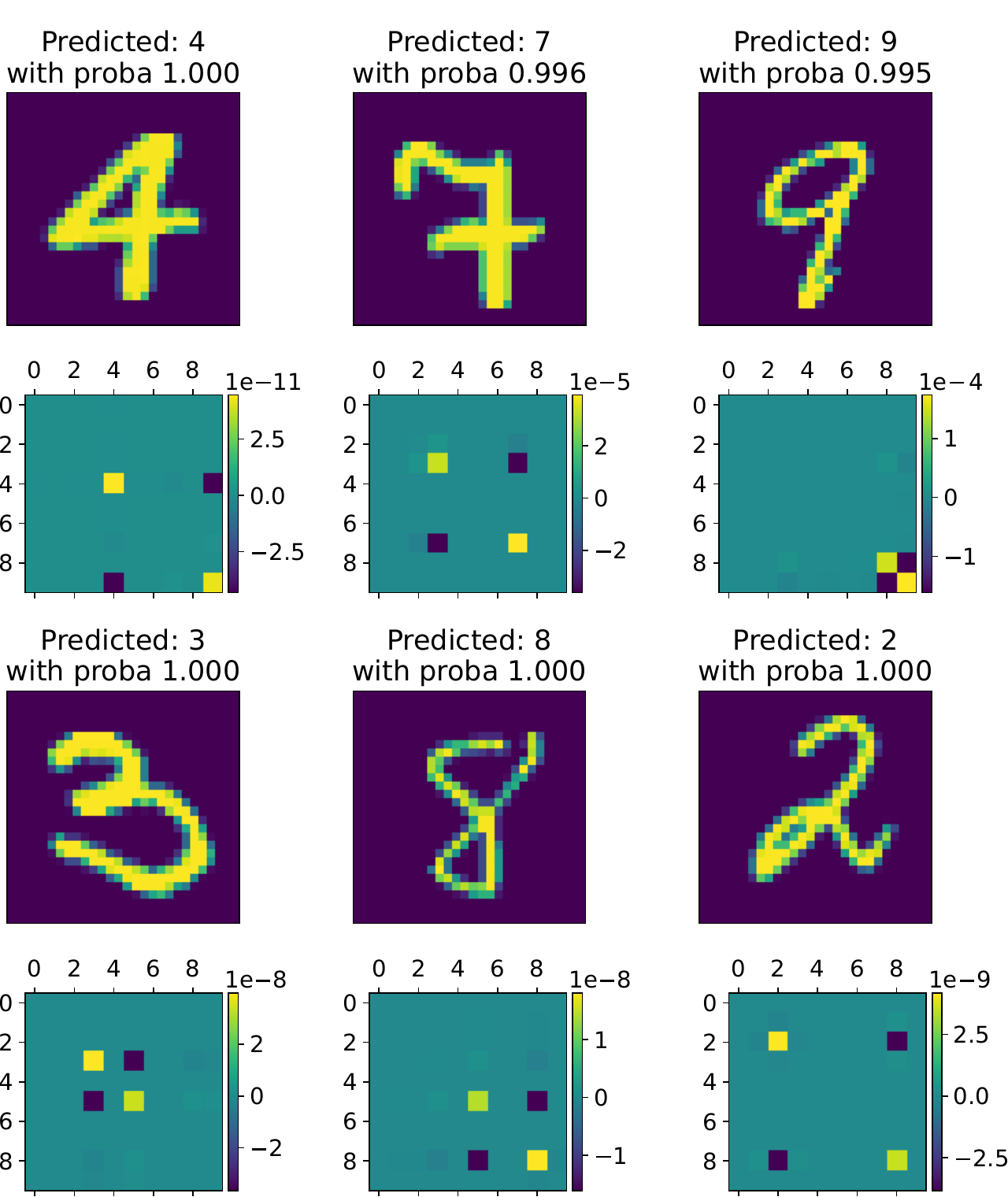}
    \caption{Couples of input point $x$ (above) and the corresponding matrix $\hat{D}(x) = \pa{g^D_x\pa{e_a,e_b}}_{a,b=1,\ldots,C}$ (below) on MNIST used for Table~\ref{tab:MNIST-experiment-seed-42} numbered from left to right and from top to bottom.}
    \label{fig:figure5}
\end{figure}

\input{Table_5}

\begin{figure}
    \centering
    \includegraphics[width=0.8\textwidth]{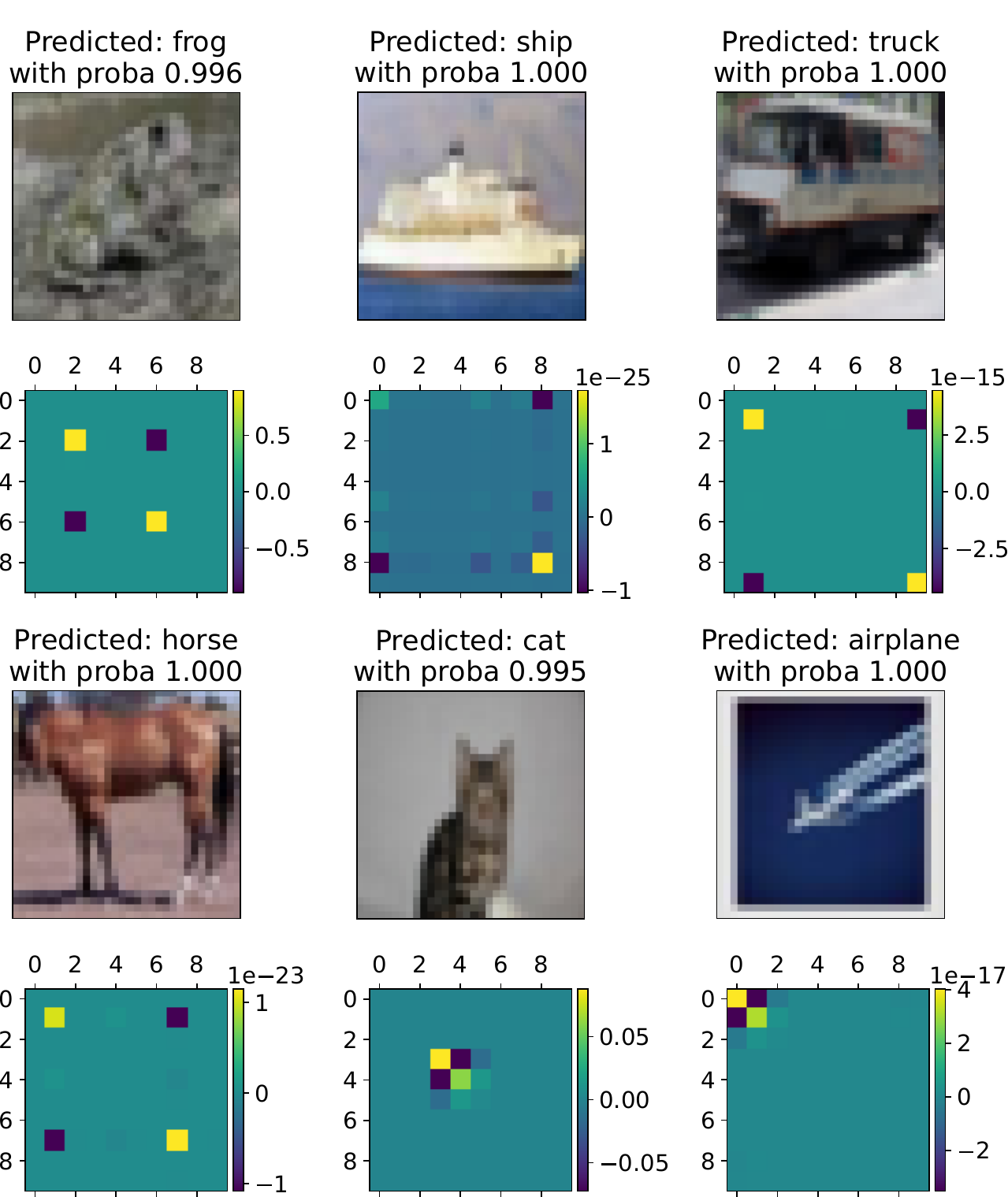}
    \caption{Couples of input point $x$ (above) and the corresponding matrix $\hat{D}(x) = \pa{g^D_x\pa{e_a,e_b}}_{a,b=1,\ldots,C}$ (below) on CIFAR10  used for Table~\ref{tab:CIFAR10-experiment-seed-42} numbered from left to right and from top to bottom.}
    \label{fig:figure6}
\end{figure}
}

\end{document}

%% file: Table_1.tex
\begin{table}[ht]
    \caption{Correspondence index - class of the CIFAR10 dataset.}
    \label{tab:CIFAR10_classes}
    \centering
    \begin{tabular}{ccccccccccc}
        \toprule
         No.  & 0 & 1 & 2 & 3 & 4 & 5 & 6 & 7 & 8 & 9 \\
         \midrule
				 Class &\textit{airplane} & \textit{automobile} & \textit{bird} & \textit{cat} & \textit{deer} & \textit{dog} & \textit{frog} & \textit{horse} & \textit{ship} & \textit{truck} \\
         \botrule
    \end{tabular}
\end{table}

%% file: Table_2.tex
\begin{table}[ht]
    \centering
    \caption{Architecture of the neural network trained on MNIST.}
    \label{tab:MNIST_architecture}
    \begin{tabular}{rl}
    \toprule
    No.  & Layers (sequential) \\
    \midrule
		(0): & \texttt{Conv2d(1, 32, kernel\_size=(3, 3), stride=(1, 1))} \\
  (1): & \texttt{ReLU()}\\
  (2): & \texttt{Conv2d(32, 64, kernel\_size=(3, 3), stride=(1, 1))}\\
  (3): & \texttt{ReLU()}\\
  (4): & \texttt{MaxPool2d(kernel\_size=2, stride=2, padding=0, dilation=1)}\\
  (5): & \texttt{Flatten()}\\
  (6): & \texttt{Linear(in\_features=9216, out\_features=128, bias=True)}\\
  (7): & \texttt{ReLU()}\\
  (8): & \texttt{Linear(in\_features=128, out\_features=10, bias=True)}\\
  (9): & \texttt{Softmax()}\\
  \botrule
    \end{tabular}
\end{table}

%% file: Table_3.tex
\begin{table}[ht]
    \caption{Architecture of the neural network trained on CIFAR10.}
    \label{tab:CIFAR10_architecture}
    \centering
    \begin{tabular}{rl}
    \toprule
    No.  & Layers (sequential)\\
    \midrule
    (0): & \texttt{Conv2d(3, 64, kernel\_size=(3, 3), stride=(1, 1), padding=(1, 1)) }\\
    (1): & \texttt{BatchNorm2d(64, eps=1e-05, momentum=0.1) }\\
    (2): & \texttt{ReLU() }\\
    (3): & \texttt{MaxPool2d(kernel\_size=2, stride=2, padding=0, dilation=1) }\\
    (4): & \texttt{Conv2d(64, 128, kernel\_size=(3, 3), stride=(1, 1), padding=(1, 1)) }\\
    (5): & \texttt{BatchNorm2d(128, eps=1e-05, momentum=0.1) }\\
    (6): & \texttt{ReLU() }\\
    (7): & \texttt{MaxPool2d(kernel\_size=2, stride=2, padding=0, dilation=1) }\\
    (8): & \texttt{Conv2d(128, 256, kernel\_size=(3, 3), stride=(1, 1), padding=(1, 1)) }\\
    (9): & \texttt{BatchNorm2d(256, eps=1e-05, momentum=0.1) }\\
    (10): & \texttt{ReLU() }\\
    (11): & \texttt{Conv2d(256, 256, kernel\_size=(3, 3), stride=(1, 1), padding=(1, 1)) }\\
    (12): & \texttt{BatchNorm2d(256, eps=1e-05, momentum=0.1) }\\
    (13): & \texttt{ReLU() }\\
    (14): & \texttt{MaxPool2d(kernel\_size=2, stride=2, padding=0, dilation=1, ceil mode=False) }\\
    (15): & \texttt{Conv2d(256, 512, kernel\_size=(3, 3), stride=(1, 1), padding=(1, 1)) }\\
    (16): & \texttt{BatchNorm2d(512, eps=1e-05, momentum=0.1) }\\
    (17): & \texttt{ReLU() }\\
    (18): & \texttt{Conv2d(512, 512, kernel\_size=(3, 3), stride=(1, 1), padding=(1, 1)) }\\
    (19): & \texttt{BatchNorm2d(512, eps=1e-05, momentum=0.1) }\\
    (20): & \texttt{ReLU() }\\
    (21): & \texttt{MaxPool2d(kernel\_size=2, stride=2, padding=0, dilation=1, ceil mode=False) }\\
    (22): & \texttt{Conv2d(512, 512, kernel\_size=(3, 3), stride=(1, 1), padding=(1, 1)) }\\
    (23): & \texttt{BatchNorm2d(512, eps=1e-05, momentum=0.1) }\\
    (24): & \texttt{ReLU() }\\
    (25): & \texttt{Conv2d(512, 512, kernel\_size=(3, 3), stride=(1, 1), padding=(1, 1)) }\\
    (26): & \texttt{BatchNorm2d(512, eps=1e-05, momentum=0.1) }\\
    (27): & \texttt{ReLU() }\\
    (28): & \texttt{MaxPool2d(kernel\_size=2, stride=2, padding=0, dilation=1, ceil mode=False) }\\
    (29): & \texttt{AvgPool2d(kernel\_size=1, stride=1, padding=0) }\\
  (30): & \texttt{Linear(in\_features=512, out\_features=10, bias=True) }\\
  (31): & \texttt{Softmax(dim=1) }\\
  \botrule
    \end{tabular}
\end{table}

%% file: Table_4.tex
\begin{table}[b]
    \centering
    \caption{Extreme values of the matrix $\hat{D}(x) = \pa{g^D_x\pa{e_a,e_b}}_{a,b=1,\ldots,C}$ evaluated on MNIST.}
    \label{tab:MNIST-experiment-seed-42}
\begin{tabular}{ccccccc}
\toprule
 & \multicolumn{2}{c}{Highest $g(e_a,e_b)$} & \multicolumn{2}{c}{Second highest $g(e_a, e_b)$} & \multicolumn{2}{c}{Lowest $g(e_a,e_b)$} \\
\cmidrule(lr){2-3} \cmidrule(lr){4-5} \cmidrule(lr){6-7}
Point & indices & value & indices & value & indices & value \\
$x$ &  $(a,b)$ & $\hat{D}(x)_{a,b}$ & $(a,b)$ & $\hat{D}(x)_{a,b}$ & $(a,b)$ & $\hat{D}(x)_{a,b}$\\
\midrule
n°0 & (4, 4) & 4.4510e-11 & (9, 9) & 4.1848e-11 & (4, 9) and (9, 4) & -4.3158e-11  \\
n°1 & (7, 7) & 3.9485e-05 & (3, 3) & 3.3201e-05 & (7, 3) and (3, 7) & -3.6186e-05  \\
n°2 & (9, 9) & 1.7384e-04 & (8, 8) & 1.4993e-04 & (9, 8) and (8, 9) & -1.6140e-04  \\
n°3 & (3, 3) & 3.9730e-08 & (5, 5) & 3.3480e-08 & (3, 5) and (5, 3) & -3.6456e-08  \\
n°4 & (8, 8) & 1.8038e-08 & (5, 5) & 1.4309e-08 & (5, 8) and (8, 5) & -1.6045e-08  \\
n°5 & (2, 2) & 4.1834e-09 & (8, 8) & 3.5215e-09 & (2, 8) and (8, 2) & -3.8379e-09  \\
\botrule
\end{tabular}
\end{table}

%% file: Table_5.tex
\begin{table}[b]
    \centering
    \caption{Extreme values of the matrix $\hat{D}(x) = \pa{g^D_x\pa{e_a,e_b}}_{a,b=1,\ldots,C}$ evaluated on CIFAR10.}
    \label{tab:CIFAR10-experiment-seed-42}
\begin{tabular}{ccccccc}
\toprule
 & \multicolumn{2}{c}{Highest $g(e_a,e_b)$} & \multicolumn{2}{c}{Second highest $g(e_a, e_b)$} & \multicolumn{2}{c}{Lowest $g(e_a,e_b)$} \\
\cmidrule(lr){2-3} \cmidrule(lr){4-5} \cmidrule(lr){6-7}
Point & indices & value & indices & value & indices & value \\
$x$ &  $(a,b)$ & $\hat{D}(x)_{a,b}$ & $(a,b)$ & $\hat{D}(x)_{a,b}$ & $(a,b)$ & $\hat{D}(x)_{a,b}$\\
\midrule
n°0 & (6, 6) & 8.9736e-01 & (2, 2) & 8.9682e-01 & (6, 2) and (2, 6) & -8.9709e-01  \\
n°1 & (8, 8) & 1.7268e-25 & (0, 0) & 6.2142e-26 & (8, 0) and (0, 8) & -1.0334e-25  \\
n°2 & (9, 9) & 4.4503e-15 & (1, 1) & 4.4414e-15 & (1, 9) and (9, 1) & -4.4458e-15  \\
n°3 & (7, 7) & 1.1507e-23 & (1, 1) & 1.0181e-23 & (7, 1) and (1, 7) & -1.0822e-23  \\
n°4 & (3, 3) & 8.8027e-02 & (4, 4) & 6.0344e-02 & (4, 3) and (3, 4) & -7.2665e-02  \\
n°5 & (0, 0) & 4.0292e-17 & (1, 1) & 3.1293e-17 & (1, 0) and (0, 1) & -3.5203e-17  \\
\botrule
\end{tabular}
\end{table}